\documentclass{article}

\usepackage{microtype}
\usepackage{graphicx}
\usepackage{subfigure}
\usepackage{booktabs} 
\usepackage{hyperref}
\usepackage{fullpage}

\usepackage{hyperref}
\hypersetup{
	colorlinks=true,
	linkcolor=blue,
	citecolor=blue,
	filecolor=blue,      
	urlcolor=blue,
}

\usepackage{amsmath}
\usepackage{amssymb}
\usepackage{mathtools}
\usepackage{amsthm}
\usepackage{float}

\usepackage{booktabs}
\usepackage{multirow}
\usepackage{diagbox}
\usepackage{hhline}
\usepackage{bbm}
\usepackage{extarrows}
\usepackage{algorithm,algpseudocode}
\usepackage{tikz}
\usetikzlibrary{matrix,arrows.meta,calc,positioning}
\usepackage{subcaption}

\usepackage[capitalize,noabbrev]{cleveref}

\numberwithin {equation} {section}

\title{BlockRR: A Unified Framework of RR-type Algorithms for Label Differential Privacy}
\author{
Haixia Liu\thanks{School of Mathematics and Statistics  \& Institute of Interdisciplinary Research for Mathematics and Applied Science \& Hubei Key Laboratory of Engineering Modeling and Scientific Computing, Huazhong University of Science and Technology, Wuhan, 430074, China. Email: liuhaixia@hust.edu.cn. The work of H.X. Liu was supported in part by Interdisciplinary Research Program of HUST 2024JCYJ005, National Key Research and Development Program of China 2023YFC3804500.},\ \   
Yi Ding\thanks{School of Mathematics and Statistics, Huazhong University of Science and Technology, Wuhan, 430074,  China. Email: m202470039@hust.edu.cn.} }

\usepackage{amsmath}
\usepackage{amssymb}
\usepackage{mathtools}
\usepackage{amsthm}

\usepackage{multirow}

\newcommand{\calS}{\mathcal{S}}
\newcommand{\bbE}{\mathbb{E}}
\newcommand{\bbI}{\mathbb{I}}
\newcommand{\calY}{\mathcal{Y}}

\allowdisplaybreaks
\numberwithin{equation}{section}

\theoremstyle{plain}
\newtheorem{theorem}{Theorem}[section]

\newtheorem{lemma}[theorem]{Lemma}

\theoremstyle{definition}
\newtheorem{definition}[theorem]{Definition}

\theoremstyle{remark}
\newtheorem{remark}[theorem]{Remark}

\begin{document}
\maketitle

\begin{abstract}
  In this paper, we introduce \textit{BlockRR}, a novel and unified randomized-response mechanism for label differential privacy. This framework generalizes existed RR-type mechanisms as special cases under specific parameter settings, which eliminates the need for separate, case-by-case analysis. Theoretically, we prove that BlockRR satisfies $\epsilon$-label DP. We also design a partition method for BlockRR based on a weight matrix derived from label prior information; the parallel composition principle ensures that the composition of two such mechanisms remains $\epsilon$-label DP. Empirically, we evaluate BlockRR on two variants of CIFAR-10 with varying degrees of class imbalance. Results show that in the high-privacy and moderate-privacy regimes ($\epsilon \leq 3.0$), our propsed method gets a better balance between test accuaracy and the average of per-class accuracy. In the low-privacy regime ($\epsilon \geq 4.0$), all methods reduce BlockRR to standard RR without additional performance loss.
\end{abstract}

\section{Introduction}
In many real-world scenarios, such as ad-click prediction or survey analysis \cite{10.1007/11761679_29, 10.1007/11681878_14}, the features (e.g., user interactions) are not sensitive, but the labels (e.g., clicks or private responses) are \cite{10.1145/2976749.2978318, 10.5555/1953048.2021036}. Label Differential Privacy (Label DP) \cite{pmlr-v19-chaudhuri11a} addresses this by applying privacy guarantees solely to the labels, unlike standard Differential Privacy (DP) \cite{dwork2006differential}, which protects both features and labels. This narrower scope allows for better model utility under the same privacy budget \cite{ghazi2021deep, esfandiari2022labeldp, Ghazi2022RegressionWL}.

The randomized response (RR) mechanism \cite{warner1965rr}—originally conceived for social surveys to mitigate response bias—is a prevalent technique in label differential privacy. Here, it functions as a label perturbation mechanism: each true label is mapped to a privatized output via a randomized process that enforces the privacy guarantee. Models are subsequently trained on these perturbed labels.

The pursuit of improved utility has led to numerous variants of the RR mechanism, tailored to different learning tasks. These include RRWithPrior \cite{ghazi2021deep} for classification, and RRonBins \cite{Ghazi2022RegressionWL} and RPWithPrior \cite{liu2025rpwithprior} for regression. Although these specialized approaches demonstrate effectiveness, the field lacks comprehensive benchmarks for direct comparison. Consequently, practitioners face a significant challenge: without quantitative evaluations, identifying the optimal algorithm for a new scenario remains difficult.

We address this gap by proposing BlockRR, a unified framework for Label DP that generalizes RR-type mechanisms. The core idea is to partition the label spaces (original and privatized) into blocks, applying tailored RR transformations between specific block pairs. This design subsumes existing methods—such as RRWithPrior and RRonBins—as special cases defined by specific hyper-parameter settings (i.e., particular block structures). Thus, BlockRR provides a common design space where algorithm comparison reduces to an analysis of hyper-parameter choices, enabling systematic benchmarking and selection. In summary, our contributions can be outlined as follows:

(1)  \textbf{We introduce a novel BlockRR mechanism.} Specifically, we partition the original label space $\calS$ into two disjoint subsets: a majority set $\calS_1$ and a minority set $\calS_2$. Using a local perturbation mapping $B(\cdot)$, we then define privatized candidate sets $\tilde{\calS}_1$ and $\tilde{\calS}_2$. This creates four mapping regions (Figure~\ref{fig:label_dp_block1}): the diagonal regions ($\calS_1 \to \tilde{\calS}_1$ and $\calS_2 \to \tilde{\calS}_2$) follow standard RR, while the off-diagonal regions use a uniform or block-uniform distribution. The block-uniform design amplifies the influence of majority labels and mitigates that of minority labels.
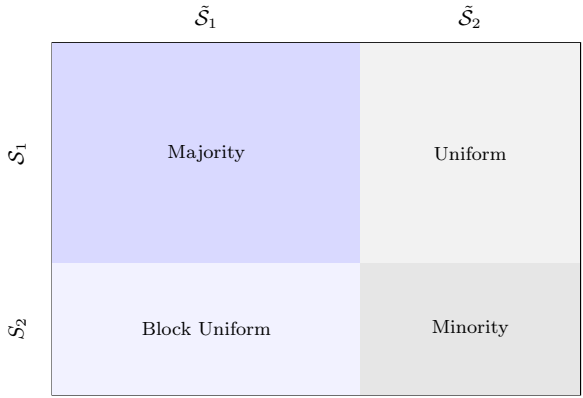
\begin{figure}[ht]
\centering
\begin{minipage}{0.48\textwidth}
    \centering
    \resizebox{\textwidth}{!}{  
    \begin{tikzpicture}[scale=1.3]

    \draw[thick] (0, 0) rectangle (6, 4);

    \draw[thick] (0, 2.5) -- (6, 2.5);  
    \draw[thick] (2.5, 0) -- (2.5, 4);  

    \fill[blue!15] (0, 1.5) rectangle (3.5, 4);  
    \fill[blue!5]  (0, 0)   rectangle (3.5, 1.5);
    \fill[gray!10]  (3.5, 1.5) rectangle (6, 4);  
    \fill[gray!20] (3.5, 0) rectangle (6, 1.5);  

    \node[rotate=90] at (-0.4, 2.75) {\small $\calS_1$};
    \node[rotate=90] at (-0.4, 0.75) {\small $S_2$};

    \node at (1.75, 4.3) {\small $\tilde{\calS}_1$};
    \node at (4.75, 4.3) {\small $\tilde{\calS}_2$};

    \node at (1.75, 2.75) {\footnotesize Majority};
    \node at (1.75, 0.75) {\footnotesize Block Uniform};
    \node at (4.75, 2.75) {\footnotesize Uniform};
    \node at (4.75, 0.75) {\footnotesize Minority};

    \end{tikzpicture}
    }
\end{minipage}    
\caption{Quadrant view of the unified label perturbation mechanism. } 
\label{fig:label_dp_block1}

\end{figure}

(2) \textbf{Our algorithm is a unified mechanism of RR-type algorithms under label DP.} Existing RR-type algorithms, including RR, RRWithPrior for classification tasks and RRonBins, RPWithPrior for regression tasks, correspond to specific configurations of BlockRR, thereby unifying them within our framework and eliminating the need for separate, case-by-case analysis.

(3) \textbf{We evaluate all methods on different tasks to guide practitioners to identify an appropriate mechanism for a new application scenario.} we evaluate BlockRR on two variants of CIFAR-10 ($\text{CIFAR-10}_1$ and $\text{CIFAR-10}_2$). Results show that in the high-privacy regime ($\epsilon \leq 3.0$), our method matches or surpasses RRWithPrior in overall accuracy while significantly mitigating its class collapse and high variance issues. At the same time, BlockRR maintains more balanced per-class performance. In the low-privacy regime ($\epsilon \geq 4.0$), all methods converge as $\calS_1 = \calS$ and $\calS_2 = \emptyset$, reducing BlockRR to standard RR.

The remainder of this paper is structured as follows: Section \ref{sec:exist_rr} reviews the concepts on differential privacy and label differential privacy and some existing RR-type algorithms. Sections \ref{sec:unified_algo} and \ref{sec:main_idea} provide a unified framework of RR-type algorithms under label DP and the main idea, respectively. Numerical implementations are discussed in Section \ref{sec:numerical}. Finally, we conclude with our findings and highlight the limitations of this work in Section \ref{sec:conclusion}. 

\section{Existing RR-type Algorithms}\label{sec:exist_rr}

This section reviews RR-type algorithms, which provide privacy guarantees for personal data. We first outline the concepts of differential privacy and label differential privacy before detailing specific algorithms.

\begin{definition}[Differential Privacy, DP \cite{dwork2006differential}]
    Let \(\epsilon, \delta \geq 0\). A randomized mechanism \(A\) is said to be \((\epsilon, \delta)\)-differentially private if, for any two adjacent datasets \(D, D'\), and for any subset \(S\) of the output space of \(A\), the probability inequality holds, \(\Pr[A(D) \in S] \leq e^\epsilon \cdot \Pr[A(D') \in S] + \delta\). If \(\delta = 0\), then the mechanism \(A\) is said to be \(\epsilon\)-differentially private (\(\epsilon\)-DP).
\end{definition} 
\begin{definition}[Label Differential Privacy, Label DP \cite{pmlr-v19-chaudhuri11a}]
    Let \(\epsilon, \delta \geq 0\). A randomized mechanism \(A\) is said to be \((\epsilon, \delta)\)-label differentially private if, for any two datasets \(D, D'\) that differ by a single label, and for any subset \(S\) of the output space of \(A\), the probability inequality holds, \(\Pr[A(D) \in S] \leq e^\epsilon \cdot \Pr[A(D') \in S] + \delta\). If \(\delta = 0\), then the mechanism \(A\) is said to be \(\epsilon\)-label differentially private (\(\epsilon\)-Label DP).
\end{definition}

Next, we introduce the RR mechanism and its key variants. 
\begin{definition}[Randomized Response, RR \cite{10.1561/0400000042}]
    Let \(K\) be a positive integer, and define \([K] := \{0,1,...,K-1\}\). Let \(\epsilon\) be a parameter, and \(y \in [K]\) be the true value, as input to \(\mathrm{RR}_\epsilon\). When an observer wants to query the value of \(y\), \(\mathrm{RR}_\epsilon\) outputs a random value \(\tilde{y} \in [K]\), which satisfies the following probability distribution:
\begin{align*}
    \Pr[\tilde{y} = \hat{y}] = 
    \begin{cases}
        \dfrac{e^\epsilon}{e^\epsilon+K-1}, & \text{for} \; \hat{y} = y, \\
        \dfrac{1}{e^\epsilon+K-1}, & \text{otherwise}.
    \end{cases}
\end{align*}
RRWithPrior is a variant of RR that uses the prior distribution of the true label \(y\) across the dataset. The algorithm outputs the top \(k\) classes with the highest probability, forming the candidate set \(\calY_k\) \((\calY_k \subseteq [K])\).
\end{definition}
\begin{definition}[Randomized Response with Prior, RRWithPrior \cite{ghazi2021deep}]
    Let \(K\) be a positive integer, and \([K] = \{1, 2, \dots, K\}\). Let \(\epsilon\) be a parameter, and \(y \in [K]\) be the true value. Let \(y \sim \mathbf{p} = [p_1\ p_2\ \dots\ p_K]\), where \(\mathbf{p}\) is the prior distribution of the true label \(y\) across the dataset. When an observer wants to query the value of \(y\), the output is a random value \(\tilde{y} \in \calY_k\) (where \(\calY_k\) consists of the indices of the top \(k\) components), satisfying the following probability distribution:
\begin{align*}
    \Pr[\tilde{y} = \hat{y} \mid y \in \calY_k] = 
    \begin{cases}
        \dfrac{e^\epsilon}{e^\epsilon+k-1}, & \text{for} \; \hat{y} = y, \\
        \dfrac{1}{e^\epsilon+k-1}, & \text{for} \; \hat{y} \in \calY_k \setminus \{y\},
    \end{cases}
\end{align*}
If $y \notin \calY_k$, it outputs a value of $\calY_k$ uniformly.
\end{definition}

While RR and RRWithPrior are primarily designed for privacy protection in classification tasks, the variants RRonbins and RPwithPrior extend the randomized-response framework to regression tasks.
\begin{definition}[Randomized Response on Bins, RRonbins \cite{Ghazi2022RegressionWL}]
    Let $\epsilon$ be a privacy parameter and $\calY$ the set of true labels. Let $\Phi: \calY \to \tilde{\calY}$ be a randomized mechanism, where $\tilde{\calY}$ is a finite set of size $|\tilde{\calY}|$. The output distribution of $\tilde{y} = \mathrm{RRonBins}(y)$ is defined as:
\begin{equation*}
    \Pr[\tilde{y} = \hat{y}] = 
    \begin{cases}
        \dfrac{e^\epsilon}{e^\epsilon + \lvert \tilde{\calY} \rvert - 1}, & \text{for} \; \hat{y} = \Phi(y), \\
        \dfrac{1}{e^\epsilon + \lvert \tilde{\calY} \rvert - 1}, & \text{otherwise}.
    \end{cases}
    \label{eq:third_eq}
\end{equation*}
\end{definition} 
\begin{definition}[Regression Privacy with Prior, RPwithPrior \cite{liu2025rpwithprior}]
    Let $\epsilon > 0$ be a privacy parameter and $\delta > 0$ a fixed positive constant. Consider a randomized mechanism $\mathcal{A}: \mathcal{Y} \to \tilde{\mathcal{Y}}$, where $y \in \mathcal{Y}$ and $\tilde{y} \in \tilde{\mathcal{Y}}$. For an optimal interval $\mathcal{I} = [A_1, A_2]$, define the neighborhoods $N_y = [y - \delta, y + \delta]$ and $N_\mathcal{I} = \bigcup_{y \in \mathcal{I}} N_y$. Let $\gamma = 2\delta + e^{-\epsilon}(A_2 - A_1)$ and $f_{\tilde{Y} \mid Y}(\tilde{y} \mid y)$ denote the conditional probability density of $\tilde{Y}$ given $Y = y$. Then for any $y \in \mathcal{I}$, $f_{\tilde{Y} \mid Y}(\tilde{y} \mid y)$ satisfies:
\[
    f_{\tilde{Y} \mid Y}(\tilde{y} \mid y) = 
    \begin{cases}
         1/\gamma, & \text{for} \; \tilde{y} \in N_y, \\
         e^{-\epsilon}/\gamma, & \text{for} \; \tilde{y} \in N_\mathcal{I} \setminus N_y, \\
         0, & \text{otherwise}.
    \end{cases}
\]
If \(y \notin \mathcal{I}\), \(f_{\tilde{Y} \mid Y}(\tilde{y} \mid y) = \dfrac{1}{2\delta + (A_2 - A_1)}, \forall\tilde{y} \in N_\mathcal{I}\).
\end{definition}

\section{Unified Framework of RR-type algorithms under Label DP}\label{sec:unified_algo}
A number of algorithms have been designed to achieve label DP. While they may seem different on the surface, these methods all share a common underlying design principle. To maximize utility, a label is mapped to itself (or a close surrogate) with high probability, with the remaining probability distributed across other labels. In the following, we introduce a unified framework for RR-type algorithms under label differential privacy.

\subsection{BlockRR}
We propose \emph{BlockRR}, a general Label DP framework designed to unify existing RR-type methods and enable more flexible randomization strategies. Its core mechanism partitions the true label set into disjoint subsets (determined in Subsection \ref{subsec:partition}). Formally, let \(\calS\) be the original label set. We partition it into two disjoint subsets, \(\calS_1\) and \(\calS_2\), where \(\calS_1\) typically contains majority or important labels and \(\calS_2\) contains minority or less important labels. Let \(\tilde{\calS} \subseteq \calS\) denote the privatized label set, and let \( B:\calS \to \tilde{\calS}\) be a randomized mapping function with $y\in B(y),\forall y\in\calS$. This partitions the output set into \(\tilde{\calS}_1 = \tilde{\calS} \cap B(\calS_1)\) and \(\tilde{\calS}_2 = \tilde{\calS} \setminus \tilde{\calS}_1\). As illustrated in Figure~\ref{fig:label_dp_block1}, the Cartesian product of the true and privatized label sets is thus divided into four regions. We proceed by constructing the overall randomized algorithm in a modular fashion, designing the specific perturbation rules for each region.

In practice, majority labels often dominate the dataset. To improve the performance of Label DP algorithms, we focus on leveraging the information-rich majority labels while appropriately mitigating the influence of minority labels, which may carry limited useful information. To this end, we introduce a subset $\Delta\subseteq\tilde{\calS}_1$ with cardinality $l$, which is a subset of the most dominated labels. 

We now define the transition probability (or probability density function) for each block. Beginning with the bottom-left block, corresponding to the pair $(\calS_2,\tilde{\calS}_1)$, we define
\begin{align*}
    &\Pr(\tilde{Y}=\tilde{y},\tilde{y}\in \Delta|Y=y,y\in \calS_2)=\frac{1}{|\tilde{\calS}|},\\ &\Pr(\tilde{Y}=\tilde{y},\tilde{y}\in\tilde{\calS}_1\setminus \Delta|Y=y,y\in \calS_2)=\beta.
\end{align*}

Similarly, for the bottom-right block, corresponding to the pair $(\calS_2,\tilde{\calS}_2)$, we define
\begin{align*}
    &\Pr(\tilde{Y}=\tilde{y},\tilde{y}\in B(y)|Y=y,y\in \calS_2)=e^\epsilon\gamma,\\ &\Pr(\tilde{Y}=\tilde{y},\tilde{y}\in\tilde{\calS}_2\setminus B(y)|Y=y,y\in \calS_2)=\gamma.
\end{align*}
For the top-left block (corresponding to the pair $\calS_1$ and $\tilde{\calS}_1$) and the top-right block (corresponding to the pair $\calS_1$ and $\tilde{\calS}_2$), the function is specified as follows: 
\begin{align*}
        &\Pr(\tilde{Y}=\tilde{y},\tilde{y}\in B(y)|Y=y,y\in \calS_1)=e^\epsilon\beta,\\ &\Pr(\tilde{Y}=\tilde{y},\tilde{y}\in\tilde{\calS}_1\setminus B(y)|Y=y,y\in \calS_1)=\beta,\\
        &\Pr(\tilde{Y}=\tilde{y},\tilde{y}\in\tilde{\calS}_2|Y=y,y\in \calS_1)=\gamma.
\end{align*}
By the probabilistic property $\bbE_{\tilde{y}}[\Pr(\tilde{Y}=\tilde{y}|Y=y)]=1,\ \forall y\in\calS$, we have the following system of linear equations
\begin{equation}\label{eq:linear_system}
    \begin{cases}
            (e^\epsilon |B(y)| + |\tilde{\calS}_1| - |B(y)| )\cdot \beta + |\tilde{\calS}_2| \cdot \gamma = 1,\ \forall y\in\calS_1, \\
            (|\tilde{\calS}_1| - |\Delta|) \cdot \beta + (e^\epsilon |B(y)| + |\tilde{\calS}_2| - |B(y)|) \cdot \gamma \\
            \hspace{4.5cm}= 1-\frac{|\Delta|}{|\tilde{\calS}|},\ \forall y\in\calS_2,
        \end{cases}
\end{equation}
where $\lvert \cdot \rvert$ is the cardinality of a set. Assuming $|B(y)|$ is constant for all $y \in \calS$ and $|\Delta| = l$, solving \eqref{eq:linear_system} yields:
\begin{equation}\label{eq:beta_gamma}
    \beta = \dfrac{\beta_1}{\kappa},\ \gamma = \dfrac{\gamma_1}{\kappa},
\end{equation}
where \begin{equation*}
\begin{split}
    \beta_1=&(e^\epsilon - 1)|B(y)| + \frac{l \cdot |\tilde{\calS}_2|}{|\tilde{\calS}|},\\
    \gamma_1=&((e^\epsilon-1)|B(y)| + l) - \frac{l}{|\tilde{\calS}|} \cdot ((e^\epsilon-1)|B(y)| + |\tilde{\calS}_1|),\\
    \kappa=&((e^\epsilon-1)|B(y)| + |\tilde{\calS}_1|) \cdot ((e^\epsilon-1)|B(y)| + |\tilde{\calS}_2|)\\
    &- (|\tilde{\calS}_1| - l) \cdot |\tilde{\calS}_2|.
    \end{split}
\end{equation*}
 For these four mapping regions (Figure~\ref{fig:label_dp_block1}), the diagonal regions ($\calS_1 \to \tilde{\calS}_1$ and $\calS_2 \to \tilde{\calS}_2$) follow standard RR, while the off-diagonal regions use a uniform or block-uniform distribution. This construction ensures both normalization, $\bbE_{\tilde{y}}[\Pr(\tilde{Y}=\tilde{y}|Y=y)]=1$ and $\epsilon$-label differential privacy. For completeness, the whole algorithm is shown in Algorithm~\ref{alg:algorithm1}, denoed as BlockRR.
\begin{algorithm}[!tbp] \caption{BlockRR$_\epsilon(y)$}
  \label{alg:algorithm1}
  \begin{algorithmic}[1]
    \State {\bfseries Input:} $\calS$, $\calS = \calS_1 \cup \calS_2$, $\calS_1 \cap \calS_2 = \emptyset$, $\tilde{\calS} \subseteq \calS$, $\tilde{\calS}_1 = \tilde{\calS} \cap B(\calS_1)$ and $\tilde{\calS}_2=\tilde{\calS}\setminus \tilde{\calS}_1$, $\Delta \subseteq \tilde{\calS}_1$, privacy parameter $\epsilon$, cardinality $l$ and  mapping $B$.
    \State Compute $\beta$ and $\gamma$ by Equation \eqref{eq:beta_gamma},
    \If{$y \in \calS_2$}
    \State Sample $\tilde{y} \in \tilde{\calS}$ according to
        \[ \Pr(\tilde{Y} = \tilde{y} \mid Y = y) = 
        \begin{cases}
        \frac{1}{|\tilde{\calS}|},&\text{if } \tilde{y}\in \Delta,\\
        e^{\epsilon} \cdot \gamma,& \text{if } \tilde{y} \in B(y), \\
        \gamma,& \text{if } \tilde{y} \notin B(y)\cup \Delta,
        \end{cases}\quad 
        \]
\ElsIf{$y \in \calS_1$}
    \If{$\tilde{y} \in \tilde{\calS}_2$}
        \State Sample with probability $\Pr(\tilde{Y} = \tilde{y} \mid Y = y) = \gamma$
    \ElsIf{$\tilde{y} \in \tilde{\calS}_1$}
        \State Sample with probability:
        \[
        \Pr(\tilde{Y} = \tilde{y} \mid Y = y) = 
        \begin{cases}
        e^{\epsilon} \cdot \beta & \text{if } \tilde{y} \in B(y), \\
        \beta & \text{if } \tilde{y} \notin B(y).
        \end{cases}\quad
\]
    \EndIf
\EndIf
\State {\bfseries Output:} Randomized labels.
  \end{algorithmic}
\end{algorithm}

Next, we will illustrate that the introduction of $\Delta\subseteq\tilde{\calS}_1$ is designed to enhance the influence of majority labels and mitigate the impact of minority labels.
\begin{lemma}\label{lem:temper}
    Under the assumptions that $|B(y)|$ is constant over $y \in \calS$ and that the partitions $\calS_1$, $\calS_2$, $\tilde{\calS}_1$, $\tilde{\calS}_2$ are fixed, $\beta$ is monotonically increasing in $l$ and $\gamma$ is monotonically decreasing in $l$. Furthermore, $\max_l\gamma=\min_l\beta$.
\end{lemma}

We leave the proof of Lemma \ref{lem:temper} to Appendix \ref{appendix:gamma-beta-proof}. 
As formalized in Lemma \ref{lem:temper}, the cardinality $l$ governs the strength of minority-label mitigation. The lemma shows that when  $l=0\  (\mathrm{i.e.}\ \Delta=\emptyset)$, mitigation is entirely absent; when $l=|\tilde{\calS}_1|\ (\mathrm{i.e.}\ \Delta=\tilde{\calS}_1)$, mitigation is strongest. Consequently, the parameter $\gamma$ is maximized and $\beta$ is minimized at $l=0$, and under the condition $\beta=\gamma$, BlockRR reduces precisely to the classical Randomized Response algorithm.

In addition, BlockRR serves as a unified framework. The following theorem demonstrates that all RR-type algorithms discussed in Section~\ref{sec:exist_rr} correspond to specific configurations of BlockRR, thereby unifying them within our framework and eliminating the need for separate, case-by-case analysis.

\begin{theorem}
     Our proposed algorithm, BlockRR, can recover various existing mechanisms under different parameter configurations:
     \begin{table}[H]
    \centering
    \caption{Unified schemes under various parameter configurations, where $\Phi$ is in Equation~\eqref{eq:third_eq}.}\vspace{-2mm}
    \label{tab:unifty}
    \resizebox{0.48\textwidth}{!}{
    \begin{tabular}{c|cccccc}
    \toprule
    \textbf{Method} & $\calS_1$ & $\calS_2$ & $\tilde{\calS}_1$ &$\tilde{\calS}_2$ & $B$ & $l$ \\
    \midrule
    \textbf{RR} & $\calS$ & $\emptyset$ & $B(\calS_1)$ & $\emptyset$ & $I$ & -- \\
    \textbf{RRWithPrior} & $\calS_1$ & $\calS_2$ & $B(\calS_1)$ & $\emptyset$ & $I$ & $|\tilde{\calS_1}|$ \\
    \textbf{RRonBins} & $\calS$ & $\emptyset$ & $B(\calS_1)$ & $\emptyset$ & $\Phi$ & -- \\
    \textbf{RPwithPrior} & $\calS_1$ & $\calS_2$ & $B(\calS_1)$ & $\emptyset$ & $x \to [x - \delta, x + \delta]$ &$|\tilde{\calS_1}|$ \\
    \bottomrule
    \end{tabular}\vspace{-2mm}
    }
    \end{table}
\end{theorem}

\begin{remark}
For RR and RRonBins, we can also understand as the following configurations:
\begin{table}[H]
    \centering
    \caption{Unified schemes under various parameter configurations.}\vspace{-2mm}
    \label{tab:unifty1}
    \begin{tabular}{c|cccccc}
    \toprule
    \textbf{Method} & $\calS_1$ & $\calS_2$ & $\tilde{\calS}_1$ &$\tilde{\calS}_2$ & $B$ & $l$ \\
    \midrule
    \textbf{RR} &$\calS_1$ & $\calS_2$ & $B(\calS_1)$ & $B(\calS_2)$ &  $I$ & 0 \\
    \textbf{RRonBins} & $\calS_1$ & $\calS_2$ & $B(\calS_1)$ & $B(\calS_2)$ & $\Phi$ & 0 \\
    \bottomrule
    \end{tabular}
    \end{table}
\end{remark}
Finally, we will prove the property of label differential privacy, whose proof is deferred to Appendix~\ref{appendix:labeldp-proof}.
\begin{theorem}\label{thm:labeldp}
BlockRR$_\epsilon$ is \(\epsilon\)-Label DP.
\end{theorem}

\subsection{Partition by Weight Matrix}\label{subsec:partition}

This section presents a method for partitioning the true label space, based on a weight matrix derived from the label prior distribution. In practice, the prior distribution $\mathbf{p}$ of the dataset is often unknown. We therefore adopt a two-stage approach: first, a subset of the training data is used to estimate the prior, the remaining data are then used for model training. Specifically, we split the dataset $D$ into two disjoint parts $D_1$ and $D_2$ randomly, and obtain the prior estimate $\mathbf{p}=[p_1\ \cdots\ p_K]^\top$ via Algorithm~\ref{alg:estimating_probability}, where $K$ is the number of classes, which corresponds to the number of bins in RRonbins and the number of intervals in RPWithPrior, respectively. 

Given the prior, we define the $(i,j)$-th entry of the weight matrix as
\begin{equation}\label{eq:weight}
    w_{ij}=p_j\exp\left(-\frac{\bbI(i\ne j)}{\sigma}\right),
\end{equation}
where $\sigma>0$ is a hyperparameter and $\bbI(\cdot)$ is an indicator function. Using the weight matrix, we partition the full label set $\calS$ into two disjoint subsets: $\calS_1=\{i:w_{ii}\ge w_{ij}\}$ and $\calS_2=\calS\setminus\calS_1$. Once the partition of $\calS$ is obtained, we can apply the BlockRR algorithm to randomize the target dataset. Algorithm~\ref{alg:weight} details the steps for generating the randomized dataset.

\begin{algorithm}[!htb]
\caption{Generation of the randomized dataset.}
\label{alg:weight}
\begin{algorithmic}[1]
\State {\bfseries Input:} $\calS$, $\tilde{\calS}$, dataset $D$, privacy parameter $\epsilon$, cardinality $l$ and  mapping $B$.
\State Split the dataset $D$ into two disjoint parts $D_1$ and $D_2$ randomly,
\State Prior estimation by $D_1$ with privacy budget $\epsilon$ via Algorithm \ref{alg:estimating_probability},
\State Compute weight matrix $W=(w_{ij})$ from Equation \eqref{eq:weight},
\State $\calS_1=\{i:w_{ii}\ge w_{ij},\forall j\in\tilde{\calS}\}$ and $\calS_2=\calS\setminus\calS_1$,
\State $\tilde{\calS}_1 = \tilde{\calS} \cap B(\calS_1)$ and $\tilde{\calS}_2=\tilde{\calS}\setminus \tilde{\calS}_1$, 
\State Compute $\Delta$ as the set of indices corresponding to the $l$ largest values of prior in $\tilde{\calS}_1$,
\State $\tilde{D}_2=\emptyset$,
\For{$(x,y)\in D_2$}
\State $\tilde{y}=\mathrm{BlockRR}_\epsilon(y)$,
\State $\tilde{D}_2=\tilde{D}_2\cup(x,\tilde{y})$.
\EndFor
\State {\bfseries Output:} $\tilde{D}_2$.
\end{algorithmic}
\end{algorithm}
Finally, we will claim that Algorithm \ref{alg:weight} is $\epsilon$-label DP, and defer the proof to Appendix \ref{appendix:proof-theo3.5}. 
\begin{theorem}\label{theo:algo-whole}
    Algorithm \ref{alg:weight} is also $\epsilon$-label DP.
\end{theorem}

\section{Main Idea of BlockRR}\label{sec:main_idea}
Traditional Randomized Response applies a one-size-fits-all perturbation strategy to all labels, disregarding dataset-specific structures during randomization. To better account for data distribution characteristics, Ghazi et al. \cite{ghazi2021deep} introduced the RRWithPrior algorithm, which incorporates prior information to adjust the randomization mechanism. A key limitation of this method, however, is its discarding of minority labels, making it unsuitable for tasks where per-class accuracy matters. To address these shortcomings, we propose the BlockRR algorithm, which introduces configurable, block-wise randomization mechanisms tailored to different regions of the label space.

The design of a suitable randomized mechanism reduces to specifying the conditional probability values $p_{\tilde{y} \mid y}$. To encode dataset-specific structure, we introduce a weight for each such conditional probability. These weights are organized into a matrix, where the entry corresponding to $(y,\tilde{y})$ is defined as follows:
\[w_{y\tilde{y}}=p_{\tilde{y}} \exp\left(-\frac{\bbI(\tilde{y}\ne y)}{\sigma}\right).\]
When the weight matrix is diagonally dominant, a standard RR mechanism is a suitable choice. However, when the weight matrix is not diagonally dominant, we partition $\calS$ into two subsets: $\calS_1=\{i:w_{ii}\ge w_{ij},\forall j\}$ and $\calS_2=\calS\setminus\calS_1$. For labels in $\calS_2$, the condition $w_{ii}\ge w_{ij},\forall j$ does not hold. Yet we still aim to keep $p_{y \mid y}$ as large as possible for all $ y\in\calS_2$ to maintain training performance. To this end, we introduce a subset $\Delta\subseteq\tilde{\calS}_1$ as the set of indices corresponding to the $l$ largest values of $\{p_i\}_{i\in\tilde{\calS}_1}$ and impose the constraint $p_{\tilde{y} \mid y}=\frac{1}{|\tilde{\calS}|},\forall y\in\calS_2$ when $\tilde{y}\in\Delta$.

We now formulate the following optimization problem:
\begin{equation}\label{eq:model}
\begin{split}
\max  &\sum_{y \in \calS,\tilde{y}\in\tilde{\calS}} w_{y\tilde{y}}p_{\tilde{y} \mid y} \\
\hbox{s.t.}&\left\{ 
\begin{array}{llll}
     p_{\tilde{y} \mid y} \leq e^\epsilon \cdot p_{\tilde{y} \mid y'}, &\forall \tilde{y}\in\tilde{\calS},\forall y, y' \in \calS, \\
    p_{\tilde{y} \mid y} = \frac{1}{|\tilde{\calS}|}, &\forall y \in \calS_2,\forall \tilde{y} \in \Delta, \\
    \sum_{\tilde{y} \in \tilde{\calS}} p_{\tilde{y} \mid y} = 1,&\forall y \in \calS, \\
    p_{\tilde{y} \mid y} \geq 0,&\forall \tilde{y}\in\tilde{\calS},\forall y \in \calS,
\end{array}
\right.
\end{split}
\end{equation}

where the first constraint enforces the $\epsilon$-label differential privacy requirement, while the third and fourth constraints arise from the probabilistic nature of the conditional distribution (non-negativity and normalization).

From the first three constraints in \eqref{eq:model}, we have the following results.  For \(y \in \calS_2\),
\begin{align}\label{eq:bound_s2}
    1 &= \sum_{\tilde{y} \in \tilde{\calS}} p_{\tilde{y} \mid y} \nonumber\\
    &=  \sum_{\tilde{y}\in B(y)}p_{\tilde{y} \mid y} + \sum_{\tilde{y} \in \tilde{\calS} \setminus (\{B(y)\}\cup\Delta)} p_{\tilde{y} \mid y} + \sum_{\tilde{y} \in\Delta} p_{\tilde{y} \mid y} \nonumber \\
    & \geq \sum_{\tilde{y}\in B(y)}p_{\tilde{y} \mid y} + e^{-\epsilon} \sum_{\tilde{y} \in \tilde{\calS} \setminus (\{B(y)\}\cup\Delta)} p_{\tilde{y} \mid \tilde{y}} + \sum_{\tilde{y} \in\Delta} \frac{1}{|\tilde{\calS}|}. 
\end{align}
For \(y \in \calS_1\),
\begin{align}\label{eq:bound_s1}
    1 =& \sum_{\tilde{y} \in \tilde{\calS}} p_{\tilde{y} \mid y}  \nonumber\\
    =&  \sum_{\tilde{y}\in B(y)}p_{\tilde{y} \mid y} + \sum_{\tilde{y} \in \tilde{\calS}_1 \setminus \{B(y)\}} p_{\tilde{y} \mid y} + \sum_{\tilde{y} \in\tilde{\calS}_2} p_{\tilde{y} \mid y} \nonumber\\
     \geq& \sum_{\tilde{y}\in B(y)}p_{\tilde{y} \mid y} + e^{-\epsilon} \sum_{\tilde{y} \in \tilde{\calS}_1 \setminus \{B(y)\}} p_{\tilde{y} \mid \tilde{y}} +e^{-\epsilon} \sum_{\tilde{y} \in\tilde{\calS}_2} p_{\tilde{y} \mid \tilde{y}}. 
\end{align}
 The equations hold for both \eqref{eq:bound_s2} and \eqref{eq:bound_s1} if and only if $p_{\tilde{y} \mid y} = e^\epsilon \cdot p_{\tilde{y} \mid y'}$, for all $y\in\calS_1, \tilde{y}\in\tilde{\calS}\setminus B(y)$ and for all $y\in\calS_2,\tilde{y}\in\tilde{\calS}\setminus(B(y)\cup\Delta)$. \eqref{eq:bound_s2} and \eqref{eq:bound_s1} together imply we have a larger feasible region for the following optimization problem compared with \eqref{eq:model}:
\begin{align}\label{eq:model1}
\max   &\sum_{y \in \calS,\tilde{y}\in\tilde{\calS}} w_{y\tilde{y}}p_{\tilde{y} \mid y} \nonumber\\
\hbox{s.t.}&\left\{ 
\begin{array}{llll}
     \sum_{\tilde{y}\in B(y)}p_{\tilde{y} \mid y} + e^{-\epsilon} \sum_{\tilde{y} \in \tilde{\calS} \setminus (\{B(y)\}\cup\Delta)} p_{\tilde{y} \mid \tilde{y}}\\
      \hspace{2.5cm}+ |\Delta|\cdot\frac{1}{|\tilde{\calS}|}\le1, \forall y\in \calS_2, \\
   \sum_{\tilde{y}\in B(y)}p_{\tilde{y} \mid y} + e^{-\epsilon} \sum_{\tilde{y} \in \tilde{\calS} \setminus \{B(y)\}} p_{\tilde{y} \mid \tilde{y}}\\
   \hspace{1.5cm}+ e^{-\epsilon}\sum_{\tilde{y} \in\tilde{\calS}_2} p_{\tilde{y} \mid \tilde{y}}\le1,\forall y \in \calS_1, \\
    p_{\tilde{y} \mid y} \geq 0,\forall \tilde{y}\in\tilde{\calS},\ \forall y \in \calS.
\end{array}
\right.
\end{align}
 We know that the optimal solution of linear programming is typically found at the boundary. That is, The equations hold for both \eqref{eq:bound_s2} and \eqref{eq:bound_s1}, which implies we have the same optimal solution for both \eqref{eq:model} and \eqref{eq:model1}. To find the optimal solution of \eqref{eq:model1} is to solve the following linear system
 \begin{equation*}%
    \begin{cases}
            \sum_{\tilde{y}\in B(y)}p_{\tilde{y} \mid y} + e^{-\epsilon} \sum_{\tilde{y} \in \tilde{\calS} \setminus (\{B(y)\}\cup\Delta)} p_{\tilde{y} \mid \tilde{y}} \\
             \hspace{3cm}+ |\Delta|\cdot\frac{1}{|\tilde{\calS}|}=1,\ \forall y\in \calS_2, \\
   \sum_{\tilde{y}\in B(y)}p_{\tilde{y} \mid y} + e^{-\epsilon} \sum_{\tilde{y} \in \tilde{\calS}_1 \setminus \{B(y)\}} p_{\tilde{y} \mid \tilde{y}} \\
    \hspace{3cm}+ e^{-\epsilon}\sum_{\tilde{y} \in\tilde{\calS}_2} p_{\tilde{y} \mid \tilde{y}}=1,\ \forall y \in \calS_1. \\
        \end{cases}
\end{equation*}

Note that $y\in B(y),\forall y\in\calS$. When $p_{\tilde{y}\mid y}=p_{y\mid y},\forall \tilde{y}\in B(y)$ is a constant, we have
 \begin{equation*}
    \begin{cases}
            |B(y)|\cdot p_{y \mid y} + e^{-\epsilon} \sum_{\tilde{y} \in \tilde{\calS} \setminus (\{B(y)\}\cup\Delta)} p_{\tilde{y} \mid \tilde{y}} \\
             \hspace{3.5cm}+ |\Delta|\cdot\frac{1}{|\tilde{\calS}|}=1,\ \forall y\in \calS_2, \\
   |B(y)|\cdot p_{y \mid y} + e^{-\epsilon} \sum_{\tilde{y} \in \tilde{\calS}_1 \setminus \{B(y)\}} p_{\tilde{y} \mid \tilde{y}}\\
    \hspace{2.5cm}+ e^{-\epsilon}\sum_{\tilde{y} \in\tilde{\calS}_2} p_{\tilde{y} \mid \tilde{y}}=1,\ \forall y \in \calS_1. \\
        \end{cases}
\end{equation*}
Note that 
$p_{\tilde{y}\mid\tilde{y}}$ is symmetric about $\tilde{y}$ on $\tilde{\calS}_1$ and $\tilde{\calS}_2$, respectively. Set $p_{\tilde{y}\mid\tilde{y}}=e^\epsilon \cdot \beta,\forall \tilde{y}\in\tilde{\calS}_1$, and $p_{\tilde{y}\mid\tilde{y}}=e^\epsilon \cdot \gamma,\forall \tilde{y}\in\tilde{\calS}_2$, we have  
\begin{equation*}
    \begin{cases}
            (e^\epsilon |B(y)| + |\tilde{\calS}_1| - |B(y)| )\cdot \beta + |\tilde{\calS}_2| \cdot \gamma = 1,\ \forall y\in\calS_1, \\
            (|\tilde{\calS}_1| - |\Delta|) \cdot \beta + (e^\epsilon |B(y)| + |\tilde{\calS}_2| - |B(y)|) \cdot \gamma \\
            \hspace{5cm}= 1-\frac{|\Delta|}{|\tilde{\calS}|},\ \forall y\in\calS_2.
        \end{cases}
\end{equation*}

\section{Numerical Experiments}\label{sec:numerical}
In this section, we empirically evaluate our proposed label differential privacy mechanism on image classification tasks, comparing it with the baseline RR and RRWithPrior mechanisms under varying privacy budgets $\epsilon$. To better reflect real-world class imbalance, we generate two derived datasets from CIFAR-10 \cite{krizhevsky2009learning} via non-uniform sampling \cite{buda2018systematic}. All experiments use a ResNet-18 architecture \cite{he2016deep}, trained from scratch without pre-trained weights.

This section proceeds as follows: we first outline the experimental setups in Subsection~\ref{subsec:exp_setup}, covering datasets and experimental details. Subsection~\ref{subsec:experimental_results} then presents comparative results under varying privacy budgets, followed by ablation analyses in Subsection~\ref{subsec:ablation}.
\subsection{Experimental Setup}\label{subsec:exp_setup}

\subsubsection{Datasets}

The CIFAR-10 dataset contains 60,000 32$\times$32 color images, divided into 10 classes, with 50,000 images used for training (5,000 per class) and 10,000 images used for testing (1,000 per class). To evaluate those methods under varied label distributions, we perform non-uniform sampling on the original CIFAR-10 training set to construct two imbalanced variants. Each derived dataset preserves the original test set but modifies the number of training samples per class as follows:

\textbf{CIFAR-10$_1$}: Training samples per class (0–9):
$[5000,4900,4700,4600,4500,4800,1000,1500,1000,1500]$.
The test set for each class contains $\frac{1}{10}$ of its corresponding training count.

\textbf{CIFAR-10$_2$}: Training samples per class (0–9): \([5000,4900,4700,4600,4500,4800,600,500,700,400]\).
Similarly, the test set for each class is scaled to $\frac{1}{10}$ of its training count.
\subsubsection{Experimental Details}\label{appendix:param_setup}
Experiments are conducted over privacy budgets $\epsilon \in \{0.6,0.8,1.0, 2.0, 3.0, 4.0, 6.0, 8.0, \infty\}$, where $\epsilon = \infty$ represents the non-private baseline (training on original labels).

For RRWithPrior and BlockRR, the label prior is estimated from the training data via a cross‑validation (CV)‑like scheme: we reserve 1\% of the training set for prior estimation (Algorithm~\ref{alg:weight}) and use the remaining 99\% for model training. Since standard RR does not require priors, it is trained on the entire training set, ensuring a fair comparison where each method uses its maximal available data.

In optimization and training, All models are trained from scratch for 160 epochs with a batch size of 512 for $\epsilon > 0.6$ and for 110 epochs when $\epsilon = 0.6$. Training uses SGD with Nesterov momentum (0.9) \cite{bottou2010large} and a cross‑entropy loss \cite{hinton1986learning}. The learning rate warms up linearly from 0 to 0.4 over the first 15\% of epochs, then decays by a factor of 10 at 30\%, 60\%, and 90\% of total training. Weight decay ($10^{-4}$) is applied for regularization. To improve generalization, we employ Mixup \cite{zhang2017mixup} with $\alpha=8.0$. 

\subsection{Experimental Results}\label{subsec:experimental_results}

In the following, we quantitatively evaluate our proposed method on Cifar-10$_1$ and Cifar-10$_2$. The experimental parameter settings are provided in Appendix~\ref{appendix:extra_para}.

For Cifar-10$_1$, Tables \ref{tab:cifar10_1} and \ref{tab:cifar10_1(avg.acc)} present the quantitative comparisons of RR, RRWithPrior, and our proposed method across privacy budgets. Table \ref{tab:cifar10_1} reports overall test accuracy, while Table \ref{tab:cifar10_1(avg.acc)} provides the average of per-class test accuracy, where per-class accuracy is deferred to \cref{appendix:per-class-acc}. 

\begin{table}[!htb]
\centering

\caption{Test Accuracy (\%) on $\text{CIFAR-10}_1$ (best in bold, second-best underlined).}\vspace{-2mm}
\label{tab:cifar10_1}
\resizebox{0.48\textwidth}{!}{
\begin{tabular}{c|ccc}
\toprule
\textbf{Privacy Budget} & \textbf{RR} & \textbf{RRWithPrior} & \textbf{Ours} \\
\midrule
$\epsilon = 0.6$ & $\underline{39.75 \pm 1.44}$ & $37.74 \pm 4.06$ & $\mathbf{40.31 \pm 2.08}$ \\
$\epsilon = 0.8$ & $46.18 \pm 1.51$ & $\mathbf{47.54 \pm 4.32}$ & $\underline{47.11 \pm 1.43}$ \\
$\epsilon = 1.0$ & $53.66 \pm 2.18$ & $\underline{54.86 \pm 4.19}$ & $\mathbf{55.12 \pm 2.69}$ \\
$\epsilon = 2.0$ & $\underline{78.74 \pm 0.48}$ & $73.20 \pm 1.53$ & $\mathbf{78.86 \pm 0.56}$ \\
$\epsilon = 3.0$ & $\mathbf{87.67 \pm 0.50}$ & $83.35 \pm 1.93$ & $\underline{87.60 \pm 0.51}$ \\
$\epsilon = 4.0$ & $90.43 \pm 0.28$ & $\mathbf{91.12 \pm 0.24}$ & $\underline{90.77 \pm 0.33}$ \\
$\epsilon = 6.0$ & $91.99 \pm 0.22$ & $\mathbf{92.61 \pm 0.21}$ & $\underline{92.10 \pm 0.30}$ \\
$\epsilon = 8.0$ & $92.18 \pm 0.28$ & $\mathbf{92.72 \pm 0.37}$ & $\underline{92.36 \pm 0.18}$ \\
$\epsilon=\infty$ & $92.25 \pm 0.33$ & $\underline{92.68 \pm 0.22}$ & $\mathbf{92.87 \pm 0.16}$ \\
\bottomrule
\end{tabular}}
\end{table}
\begin{table}[!htb]
\centering
\caption{The average of per-class accuracy over 10 classes (\%) on $\text{CIFAR-10}_1$ (best in bold, second-best underlined).}\vspace{-2mm}
\label{tab:cifar10_1(avg.acc)}
\begin{tabular}{c|ccc}
\toprule
Privacy Budget & RR & RRWithPrior & Ours \\
\midrule
$\epsilon=0.6$ & $\mathbf{29.14}$ & 26.17 & $\underline{28.52}$ \\
$\epsilon=0.8$ & $\mathbf{37.25}$ & 33.07 & $\underline{35.40}$ \\
$\epsilon=1.0$ & $\mathbf{46.01}$ & 38.42 & $\underline{43.30}$ \\
$\epsilon=2.0$ & $\mathbf{75.96}$ & 50.08 & $\underline{74.03}$ \\
$\epsilon=3.0$ & $\mathbf{87.21}$ & 70.85 & $\underline{86.09}$ \\
$\epsilon=4.0$ & 89.37 & $\mathbf{90.19}$ & $\underline{89.80}$ \\
$\epsilon=6.0$ & $91.08$ & $\mathbf{91.81}$ & \underline{91.22} \\
$\epsilon=8.0$ & 91.42 & $\mathbf{91.59}$ & $\underline{91.53}$ \\
$\epsilon=\infty$ & 91.54 & $\underline{91.77}$ & $\mathbf{91.88}$ \\
\bottomrule
\end{tabular}
\end{table}
 
The experimental results in Tables \ref{tab:cifar10_1} and \ref{tab:cifar10_1(avg.acc)} illustrate method performance across various privacy budgets. As noted, all methods converge to similar performance at $\epsilon \ge 4.0$, consistent with our framework where $\calS_1 = \calS$ and $\calS_2 = \emptyset$, degrading all mechanisms to standard RR.

In the high-privacy regime ($\epsilon = 0.6,0.8, 1.0$), our BlockRR mechanism demonstrates a clear advantage, achieving accuracies of 40.31\% ($\epsilon=0.6$) and 55.12\% ($\epsilon=1.0$), respectively. For the average of per-class accuracy reported in \cref{tab:cifar10_1(avg.acc)}, our proposed mechanism is superior to RRWithPrior because it avoids effectively ``class collapse''—where some classes achieve near-zero accuracy—while maintaining competitive overall accuracy compared to baseline RR. Morover, our propsed method gets a better balance between test accuracy and the average of per-class accuracy. 

For moderate budgets ($\epsilon = 2.0, 3.0$), our method
achieves the best performance with smaller variance.

\begin{table}[!htb]
\centering
\caption{Test Accuracy(\%) on $\text{CIFAR-10}_2$ (best in bold, second-best underlined).}\vspace{-2mm}
\label{tab:cifar10_2}
\resizebox{0.48\textwidth}{!}{
\begin{tabular}{c|ccc}
\toprule
\textbf{Privacy Budget} & \textbf{RR} & \textbf{RRWithPrior} & \textbf{Ours} \\
\midrule
$\epsilon = 0.6$ & $42.55 \pm 1.71$ & $\underline{42.70 \pm 4.62}$ & $\mathbf{43.48 \pm 1.98}$ \\
$\epsilon = 0.8$ & $48.70 \pm 2.59$ & $\mathbf{52.12 \pm 3.57}$ & $\underline{51.01 \pm 1.29}$ \\
$\epsilon = 1.0$ & $55.82 \pm 2.33$ & $\mathbf{59.48 \pm 5.94}$ & $\underline{58.26 \pm 1.67}$ \\
$\epsilon = 2.0$ & $\underline{79.09 \pm 0.82}$ & $78.91 \pm 0.76$ & $\mathbf{79.14 \pm 0.44}$ \\
$\epsilon = 3.0$ & $\mathbf{87.52 \pm 0.32}$ & $84.05 \pm 1.93$ & $\underline{87.18 \pm 0.30}$ \\
$\epsilon = 4.0$ & $\underline{90.22  \pm 0.17}$ & $88.65 \pm 1.42$ & $\mathbf{90.36 \pm 0.33}$ \\
$\epsilon = 6.0$ & $91.95 \pm 0.28$ & $\mathbf{92.44 \pm 0.34}$ & $\underline{92.35 \pm 0.23}$ \\
$\epsilon = 8.0$ & $92.20 \pm 0.17$ & $\mathbf{92.67 \pm 0.19}$ & $\underline{92.58 \pm 0.26}$ \\
$\epsilon = \infty$ & $92.30 \pm 0.25$ & $\mathbf{92.92 \pm 0.34}$ & $\underline{92.58 \pm 0.23}$ \\
\bottomrule
\end{tabular}
}
\end{table}
\begin{table}[!htb]
\centering
\caption{The average of per-class accuracy over 10 classes (\%) on $\text{CIFAR-10}_2$ (best in bold, second-best underlined).}\vspace{-2mm}
\label{tab:cifar10_2(avg.acc)}
\begin{tabular}{c|ccc}
\toprule
Privacy Budget & RR & RRWithPrior & Ours \\
\midrule
$\epsilon=0.6$ & $\underline{27.43}$ & 27.24 & $\mathbf{27.82}$ \\
$\epsilon=0.8$ & 33.08 & $\mathbf{33.48}$ & $\underline{33.24}$ \\
$\epsilon=1.0$ & $\mathbf{45.95}$ & 38.42 & $\underline{43.30}$ \\
$\epsilon=2.0$ & $\mathbf{67.96}$ & 50.87 & $\underline{60.09}$ \\
$\epsilon=3.0$ & $\mathbf{81.67}$ & 54.94 & $\underline{80.90}$ \\
$\epsilon=4.0$ & $\mathbf{89.24}$ & 74.01 & $\underline{86.01}$ \\
$\epsilon=6.0$ & 88.88 & $\mathbf{91.81}$ & $\underline{89.30}$ \\
$\epsilon=8.0$ & 89.38 & $\mathbf{91.59}$ & $\underline{89.69}$ \\
$\epsilon=\infty$ & 89.08 & $\mathbf{90.16}$ & $\underline{89.69}$ \\
\bottomrule
\end{tabular}
\end{table}
\begin{table*}[!htb]
    \centering
    \caption{Ablation results of the parameter $l$ under different global privacy budgets $\epsilon$ (test accuracy \%).}\vspace{-2mm}
    \label{tab:delta-ablation}
    \resizebox{\linewidth}{!}{
    \begin{tabular}{c|ccccccc}
        \toprule
        $l$ & 0 & 1 & 2 & 3 & 4 & 5 & 6\\
        \midrule
        $\epsilon=0.6$   & $42.52 \pm 1.88$ & $42.78 \pm 2.47$& $\mathbf{43.48 \pm 1.98}$ & $41.52 \pm 2.46$ & $43.22 \pm 1.51$ & $42.95 \pm 1.48$ & $42.34 \pm 1.62$\\
        $\epsilon=0.8$   & $48.88 \pm 1.94$ & $48.79 \pm 1.98$ & $49.91 \pm 1.90$
        & $49.90 \pm 2.93$ & $49.44 \pm 2.31$ & $48.80 \pm 1.80$ & $\mathbf{51.01 \pm 1.29}$\\
        $\epsilon=1.0$   &$55.58 \pm 1.91$ & $55.65 \pm 1.25$ & $57.10 \pm 1.91$ & $56.40 \pm 1.59$ & $56.04 \pm 1.71$ &$\mathbf{58.26 \pm 1.67}$ & $56.92 \pm 2.85$\\
        $\epsilon=2.0$ & $79.03 \pm 0.79$ & $78.61 \pm 0.67$ & $78.76 \pm 0.78$ & $78.77 \pm 0.92$ & $78.54 \pm 0.74$ & $78.72 \pm 1.05$ & $\mathbf{79.14 \pm 0.44}$\\
        $\epsilon=3.0$   &$\mathbf{87.18 \pm 0.30}$ &$87.05 \pm 0.33$  &$86.99 \pm 0.48$  &$86.78 \pm 0.42$ &$86.59 \pm 0.39$  &$86.09 \pm 0.44$ &$85.94 \pm 0.35$ \\
        $\epsilon=4.0$  & $\mathbf{90.36  \pm 0.32}$ &$90.09 \pm 0.42$ & $89.99 \pm 0.27$ & $89.85 \pm 0.60$ & $89.62 \pm 0.38 $ & $88.96 \pm 0.35$ & $88.78 \pm 0.53$\\
        \bottomrule
    \end{tabular}
    }
\end{table*}
Next, we evaluate those mechanisms on $\text{CIFAR-10}_2$, reporting overall accuracy in Table~\ref{tab:cifar10_2} and average of per-class accuracy in Table~\ref{tab:cifar10_2(avg.acc)}, where per-class accuracy is deferred to \cref{appendix:per-class-acc}. In the high-privacy regime ($\epsilon \leq 1.0$), the stronger label imbalance allows both RRWithPrior and our method to outperform RR in overall accuracy. Crucially, our mechanism shows significantly lower variance than RRWithPrior, highlighting its stability. Furthermore, it effectively avoids the class collapse observed in RRWithPrior at low $\epsilon$ and maintains balanced per-class performance for all $\epsilon \geq 1.0$.

\subsection{Ablation Study}
\label{subsec:ablation}
In the ablation study, we aim to analyze the impact of the hyperparameters $l$ and $\sigma$ on the
performance of our proposed BlockRR mechanism. The following experiments are all conducted on Cifar-10$_2$.

\subsubsection{Effectness of \texorpdfstring{$l$}{l}}
\label{subsubsec:l}
As demonstrated in Subsection \ref{subsec:experimental_results}, the performance of our proposed mechanism is almost the same as the RR mechanism at $\epsilon > 4.0$, since most of the partitions become $\calS_1 = \calS$ and $\calS_2 = \emptyset$. Consequently, the hyperparameter $l$ only affect elements in $\calS_2$, its specific value becomes irrelevant in this regime. 

In the following, we conduct experiments across different privacy budgets $\epsilon \in \{0.6, 0.8, 1.0, 2.0, 3.0, 4.0\}$ and vary $l \in \{0,1,2,3,4,5,6\}$. The resulting test accuracies are reported in Table~\ref{tab:delta-ablation}.

As shown in Table~\ref{tab:delta-ablation}, test accuracy remains stable for high privacy budgets ($\epsilon \le 1.0$), fluctuating within 2\% as $l$ varies from 0 to 6. For instance, at $\epsilon = 0.6$, accuracy ranges from 41.52\% ($l=3$) to 43.44\% ($l=2$). From the viewpoint of the transition probability matrix, according to the definition of $l$ in our proposed method, smaller values of $l$ make the transition probabilities closer to those of RR (and when $l=0$ the mechanism reduces exactly to RR), while larger values of $l$ move the mechanism towards RRWithPrior. The fact that the optimal $l$ is not at the boundaries (e.g., $l=2$ for $\epsilon=0.6$) indicates that the best performance is achieved by an intermediate configuration, blending properties of both RR and RRWithPrior.

When $\epsilon$ increases to 3.0, overall accuracy  shows a mild negative trend with $l$, dropping from 87.18\% ($l=0$) to 85.94\% ($l=6$), a decrease of 1.2\%. This aligns with the results in Table~\ref{tab:cifar10_2} for $\epsilon = 3.0$: under higher privacy budgets, better performance is achieved when the transition matrix is closer to the standard RR mechanism, which itself is a special case of our proposed framework when $l=0$. We have a similar result for $\epsilon=4.0$.

\subsubsection{Sensitivity to \texorpdfstring{$\sigma$}{sigma}}
\label{subsubsec:sigma}
\begin{figure}[!htb]
    \centering
        \begin{minipage}{0.49\linewidth}
            \includegraphics[width=0.92\linewidth]{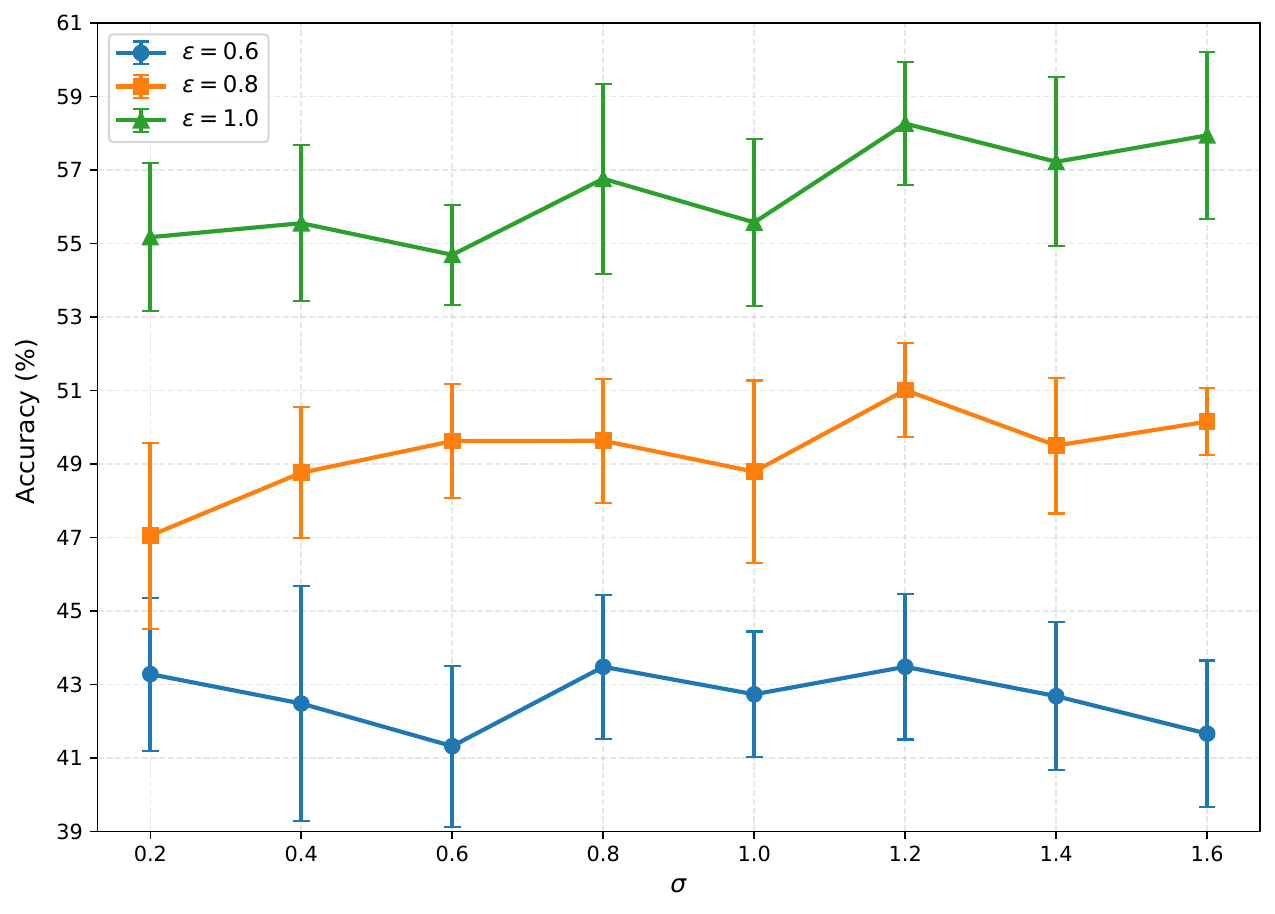}\vspace{-3mm}
        \caption{Accuracy (\%) versus $\sigma$ with $\epsilon\in\{0.6,0.8,1.0\}$.}
        \label{fig:acc_a}
        \end{minipage}
        \begin{minipage}{0.49\linewidth}
            \includegraphics[width=0.92\linewidth]{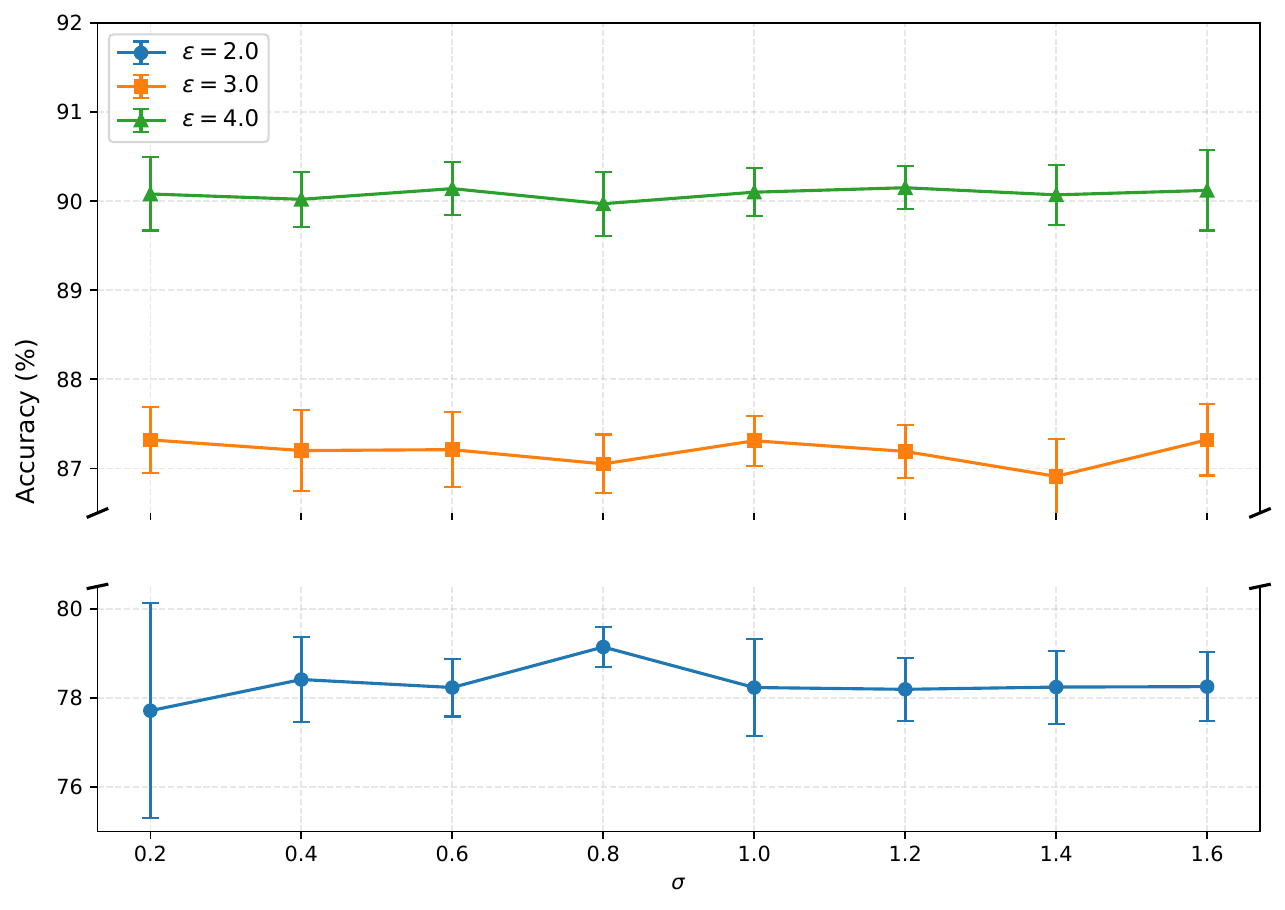}\vspace{-3mm}
        \caption{Accuracy (\%) versus $\sigma$ with $\epsilon\in\{2.0,3.0,4.0\}$.}
        \label{fig:acc_b}
        \end{minipage}
\end{figure}

We study the sensitivity of our method to the hyperparameter $\sigma$ by varying $\sigma \in [0.2, 1.6]$ and report the overall accuracy under different privacy budgets $\epsilon$. Overall, the performance is stable across a wide range of $\sigma$, indicating that our method is not overly sensitive to this hyperparameter.

In the strict privacy regime (\cref{fig:acc_a}, $\epsilon \le 1.0$), accuracy exhibits mild fluctuations as $\sigma$ changes. For all three budgets ($\epsilon=0.6, 0.8, 1.0$), very small $\sigma$ values tend to yield slightly lower accuracy, while increasing $\sigma$ to a moderate range generally improves performance and the best results are typically achieved around $\sigma =1.2$. 

In the moderate privacy budget regime (\cref{fig:acc_b}, $\epsilon \ge 2.0$), the curves become even flatter. For $\epsilon=3.0$ and $\epsilon=4.0$, accuracy remains almost constant across the entire $\sigma$ range, with variations that are comparable to the error bars. These results confirm that once the privacy budget is less restrictive, the choice of $\sigma$ has minimal impact on performance.

\section{Conclusion}
\label{sec:conclusion}

In this paper, we introduce BlockRR, a novel and unified RR-type mechanism for label differential privacy. This framework generalizes existed RR-type mechanisms as special cases under specific parameter settings. Theoretically, we prove that BlockRR satisfies $\epsilon$-label DP. We also design a partition method for BlockRR based on a weight matrix derived from label prior information; the parallel composition principle ensures that the composition of two such mechanisms remains $\epsilon$-label DP. Empirically, we evaluate BlockRR on two variants of CIFAR-10 ($\text{CIFAR-10}_1$ and $\text{CIFAR-10}_2$) with varying degrees of class imbalance. Results show that in the high-privacy regime ($\epsilon \leq 1.0$), our method significantly mitigates its class collapse and high variance issues. In the low-privacy regime ($\epsilon \geq 4.0$), all methods reduce BlockRR to standard RR (when $l=0$) without additional performance loss. In addition, our propsed method shows a better balance between test accuaracy and the average of per-class accuracy. While there is extensive theoretical work on label-flipping noise, a comprehensive understanding of block-wise flipping noise remains an important direction for future research.

\bibliographystyle{alpha}
\bibliography{ref}

\newpage
\appendix
\section{Related Work}\label{sec:related}

Over the past decade, the field of differential privacy has seen significant theoretical and practical advances. A cornerstone technique is RR \cite{warner1965rr}, introduced by Warner in 1965. Originally designed for anonymous surveys, its core idea is to allow participants to provide untruthful answers with a known probability when reporting sensitive data. This preserves individual privacy while still enabling accurate statistical analysis. With the formalization of the differential privacy framework, RR has become a crucial tool in DP research. Specifically, RR methods can be applied directly during data collection, preventing raw sensitive data from ever being exposed to a central server and thus providing a strong privacy guarantee \cite{10.1145/2660267.2660348}.

Additive noise mechanisms form another essential category of differential privacy methods. The Laplace Mechanism \cite{10.1007/11681878_14} achieves privacy by perturbing query outputs with Laplace-distributed noise. Similarly, the Gaussian Mechanism \cite{10.1561/0400000042} adds noise from a Gaussian distribution. A distinct variant, the Staircase Mechanism \cite{kairouz2014extremal}, ensures privacy by discretizing outputs into stepped values and injecting noise within this discrete structure.

Under the standard sample-level differential privacy (DP) framework, DP-ERM \cite{10.5555/1953048.2021036} enables the private training of regularized empirical risk minimization by the output perturbation or the objective perturbation. Initial approaches to private training focused on Stochastic Gradient Descent (SGD), privatizing each parameter update by injecting calibrated noise directly into the computed gradients \cite{6736861}. This method evolved into the now-standard DP-SGD algorithm, which enforces privacy by clipping individual per-sample gradients and adding Gaussian noise to their aggregated batch averages \cite{10.1145/2976749.2978318}. Beyond these foundational techniques, research has extended privacy mechanisms into deep learning architectures. For instance, in Convolutional Neural Networks (CNNs), randomization layers have been successfully integrated to perturb internal feature representations directly \cite{8894030}.

To further enhance privacy protection,  Chaudhuri \& Hsu \cite{pmlr-v19-chaudhuri11a} introduced the concept of Label Differential Privacy (Label DP), which specifically perturbs and protects only the data labels while keeping the features unchanged. Research into deep learning under this framework has advanced significantly. For instance,  Malek Esmaeili et al \cite{NEURIPS2021_37ecd276} proposed the ALIBI method, which integrates Laplace noise addition with Bayesian denoising, and also adapted the PATE framework to Label DP by combining it with semi-supervised learning \cite{papernot2017semisupervised}. In addition, Ghazi et al. \cite{ghazi2021deep} built upon Randomized Response (RR) techniques. They enhanced the RRTop-$k$ method by filtering out minority class samples and upsampling majority classes. Additionally, they introduced an RRWithPrior mechanism that leverages prior knowledge of the data distribution to mitigate the utility loss. 

Label DP has also gained significant attention in regression tasks. Expanding on their previous work, Ghazi et al. \cite{Ghazi2022RegressionWL} proposed RRonbins, a Label DP randomization mechanism based on global prior distributions. This method provides privacy by having labels ``randomly respond" to discretized value intervals (bins), and employs an efficient algorithm to determine the optimal bin size. Furthermore, they combined an optimal unbiased randomization scheme, where the perturbed regression labels preserves statistical unbiasedness \cite{NEURIPS2023_bd5d4366}.

\section{Proof of Lemma \ref{lem:temper}}\label{appendix:gamma-beta-proof}
\begin{proof}
Since the numerator of $\beta$ increases with $l$ and its denominator decreases with $l$, $\beta$ is monotonically increasing in $l$. From the expression for $\gamma$,
\[\gamma=\frac{1}{|\tilde{\calS}_2|}(1-(e^\epsilon |B(y)| + |\tilde{\calS}_1| - |B(y)| )\cdot \beta),\]
it follows that $\gamma$ is monotonically decreasing in $l$. Consequently, $\beta$ attains its minimum and $\gamma$ gets its maximum at $l=0$, and $\max_l \gamma = \min_l \beta$.
\end{proof}

\section{Proof of \texorpdfstring{$\epsilon$}{}-LabelDP}
\label{appendix:labeldp-proof}

\begin{proof}[Proof of Theorem \ref{thm:labeldp}.]
We only need to prove that for any $y, y' \in \calS$ and any $\tilde{y} \in \tilde{\calS}$, the following holds:
\begin{align}
    \Pr[\tilde{Y} = \tilde{y} \mid Y = y] \leq e^\epsilon \Pr[\tilde{Y} = \tilde{y} \mid Y = y'].
    \label{eq:labeldp}
\end{align}

For any $\tilde{y} \in \tilde{\calS}_1$ and any $y \in \calS$,
\[
\Pr[\tilde{Y} = \tilde{y} \mid Y = y] =
\begin{cases}
    e^\epsilon \cdot \beta, & \text{for } y \in \calS_1, \tilde{y} \in B(y), \\
    \beta, & \text{for } y \in \calS_1, \tilde{y} \in \tilde{\calS}_1 \setminus B(y) \text{ or } y \in \calS_2, \tilde{y} \in \tilde{\calS}_1 \setminus \Delta,\\
    \dfrac{1}{|\tilde{\calS}|}, & \text{for } y \in \calS_2, \tilde{y} \in \Delta,
\end{cases}
\]

where $\beta = \dfrac{(e^\epsilon  - 1)|B(y)| + \frac{l|\tilde{\calS}_2|}{|\tilde{\calS}|}}{(e^\epsilon|B(y)| + |\tilde{\calS}_1| - |B(y)|)(e^\epsilon |B(y)| + |\tilde{\calS}_2| - |B(y)|) - (|\tilde{\calS}_1| - l) |\tilde{\calS}_2|}$.

We only need to verify that $\beta \leq \dfrac{1}{|\tilde{\calS}|} \leq e^\epsilon \beta$. We have
\begin{align*}
    &\quad\;\; \dfrac{1}{|\tilde{\calS}|} - \beta \geq 0 \\
    &\Leftrightarrow \dfrac{1}{|\tilde{\calS}|} - \dfrac{(e^\epsilon - 1)|B(y)| + \frac{l|\tilde{\calS}_2|}{|\tilde{\calS}|}}{(e^\epsilon|B(y)| + |\tilde{\calS}_1| - |B(y)|)(e^\epsilon|B(y)| + |\tilde{\calS}_2| - |B(y)|) - (|\tilde{\calS}_1| - l) |\tilde{\calS}_2|} \geq 0 \\
    &\Leftrightarrow (e^\epsilon|B(y)| + |\tilde{\calS}_1| - |B(y)|)(e^\epsilon|B(y)| + |\tilde{\calS}_2| - |B(y)|) - (|\tilde{\calS}_1| - l) |\tilde{\calS}_2|\geq (e^\epsilon|B(y)| - |B(y)|) |\tilde{\calS}| - l |\tilde{\calS}_2|\\
    &\Leftrightarrow (e^\epsilon|B(y)| - |B(y)|)^2 \geq 0.
\end{align*}

Next, we check the other inequality:
\begin{align*}
    &\quad\;\; e^\epsilon \beta - \dfrac{1}{|\tilde{\calS}|} \geq 0 \\
    &\Leftrightarrow \dfrac{e^\epsilon \left( (e^\epsilon - 1)|B(y)| + \frac{l |\tilde{\calS}_2|}{|\tilde{\calS}|} \right)}{(e^\epsilon|B(y)| + |\tilde{\calS}_1| - |B(y)|)(e^\epsilon|B(y)| + |\tilde{\calS}_2| - |B(y)|) - (|\tilde{\calS}_1| - l) |\tilde{\calS}_2|} - \dfrac{1}{|\tilde{\calS}|} \geq 0 \\
    &\Leftrightarrow e^\epsilon \left( (e^\epsilon|B(y)| - |B(y)|) + \frac{l |\tilde{\calS}_2|}{|\tilde{\calS}|} \right) |\tilde{\calS}| \geq\\
    &\quad\;\;(e^\epsilon|B(y)| + |\tilde{\calS}_1| - |B(y)|)(e^\epsilon|B(y)| + |\tilde{\calS}_2| - |B(y)|) - (|\tilde{\calS}_1| - l) |\tilde{\calS}_2|\\
    &\Leftrightarrow  (e^\epsilon - 1) ^2(|\tilde{\calS}| - |B(y)|)|B(y)| + (e^\epsilon - 1) l |\tilde{\calS}_2| \geq 0,
\end{align*}
where the last inequality is from $\epsilon > 0$ and the set $\tilde{\calS}$ contains $B(y)$.

For any $\tilde{y} \in \tilde{\calS}_2$ (assume $\tilde{\calS}_2 \neq \emptyset$) and $y \in \calS$,
\[
\Pr[\tilde{Y} = \tilde{y} \mid Y = y] =
\begin{cases}
    e^\epsilon \cdot \gamma, & \text{for } y \in \calS_2, \tilde{y} \in B(y), \\
    \gamma, & \text{otherwise.}
\end{cases}
\]
where $$\gamma = \dfrac{(e^\epsilon|B(y)| + l - |B(y)|) - \frac{l}{|\tilde{\calS}|}(e^\epsilon|B(y)| + |\tilde{\calS}_1| - |B(y)|)}{(e^\epsilon|B(y)| + |\tilde{\calS}_1| - |B(y)|) \cdot (e^\epsilon|B(y)| + |\tilde{\calS}_2| - |B(y)|) - (|\tilde{\calS}_1| - l)|\tilde{\calS}_2|},$$ which clearly satisfies \eqref{eq:labeldp}. Thus, we conclude that $\mathrm{BlockRR}_\epsilon$ satisfies $\epsilon$-Label DP.
\end{proof}
\section{Proof of Theorem \ref{theo:algo-whole}}\label{appendix:proof-theo3.5}
\begin{proof}
    Both Algorithms \ref{alg:algorithm1} and \ref{alg:estimating_probability} satisfy $\epsilon$-label differential privacy. By the parallel composition principle stated in Lemma \ref{lem:parallel}, Algorithm \ref{alg:weight}, which composes these two mechanisms, also satisfies $\epsilon$-label differential privacy.
\end{proof}
\section{Laplace Mechanism for Estimating Probalility Distribution}
We begin by recalling the Laplace mechanism for probability distribution estimation, as described in \cite{Ghazi2022RegressionWL}.
\begin{algorithm}[H]
\caption{Laplace Mechanism for Estimating Probalility Distribution $M_{\epsilon}^{\mathrm{Lap}}$}\label{alg:estimating_probability}
\begin{algorithmic}[1]
\State {\bfseries Input:} Privacy parameter $\epsilon \geq 0$, $D_1 = \{(x_i, y_i)\}_{i=1}^n,$

\For{$(x,y) \in D_1$}
    \State $h_y \leftarrow \text{number of } i \text{ such that } y_i = y$,
    \State $h'_y \leftarrow \max\{h_y + \text{Lap}(2/\epsilon), 0\}$,
\EndFor

\State {\bfseries Output:} Distribution $\mathbf{p}$ over $\mathcal{{Y}}$ such that $p_y = \frac{h'_y}{\sum_{(x,y) \in D_1}h'_y}$.
\end{algorithmic}
\end{algorithm}

\section{Parallel Composition Theorem}
In the following, we recall the parallel composition theorem in differential privacy.

Let \( \{U_i\}_{i=1}^k \) be a partition of a domain $U$, i.e., the $U_i$ are pairwise disjoint and \( \cup_{i=1}^kU_i=U \). For any dataset $D\subseteq U$, the sets $\{D\cap U_i\}^k_{i=1}$ constitute a partition of $D$. Then, the parallel composition theorem can be stated as follows:
\begin{lemma}[Parallel Composition \cite{10.1145/1559845.1559850}]\label{lem:parallel}
Suppose that \( \{M_i: U \rightarrow R_i \}_{i=1}^k \) is a family of randomized algorithms, each of which satisfies $\epsilon$-DP, and that \( \{ U_i \}^{k}_{i=1} \) is a partition of the domain $U$. Then the combined mechanism $M(D) := \{M_i(D \cap U_i)\}_{i=1}^k$ still satisfies $\epsilon$-DP.
\end{lemma}

\begin{proof}
    For any adjacent datasets \( D, D' \subseteq U \), let \( D \cap U_i = D_i \) and \( D' \cap U_i = D'_i \). For any \( r_i \in R_i \), let $r = [r_1\  r_2\ \ldots\ r_k]^\top$. Then
    \[
        \Pr [M(D) = r] = \prod_{i=1}^{k} \Pr [M_i(D_i) = r_i].
    \]

    In particular, since \( D \) and \( D' \) are adjacent datasets, there is exactly one pair in \( \{(D_i, D'_i)\}_{i=1}^k \) that are adjacent (i.e., differ in exactly one record), while all the other pairs are equal. Hence,
    \begin{align*}
        \Pr [M(D) = r] 
        = \prod_{i=1}^{k} \Pr [M_i(D_i) = r_i] \leq e^\epsilon \prod_{i=1}^{k} \Pr [M_i(D'_i) = r_i] = e^\epsilon \cdot \Pr [M(D') = r]. 
    \end{align*}
    Therefore, the parallel composition satisfies $\epsilon$-DP.
\end{proof}

\section{Extra Results}
\label{appendix:extra results}

\subsection{Parameter Setup}\label{appendix:extra_para}
The specific privacy parameter settings for $\sigma$ and $l$ are reported in Tables~\ref{tab:dp_params1} and ~\ref{tab:dp_params2}.

\begin{table}[htbp]
  \centering
  \caption{Parameter settings for various privacy budgets \( \epsilon \) on CIFAR-10\(_1\)}\vspace{-3mm}
  \label{tab:dp_params1}
  \begin{tabular}{c|c c c c c c c c}
    \toprule
    $\epsilon$ & 0.6 & 0.8 & 1 & 2 & 3 & 4 & 6 & 8 \\
    $\sigma$  & 1.40 & 1.20 & 1.20 & 0.80 & 0.78 & 0.50 & 0.50 & 0.50 \\
    $l$       & 6 & 5 & 5 & 2 & 0 & 0 & -- & -- \\
    \bottomrule
  \end{tabular}
\end{table}

\begin{table}[htbp]
  \centering
  \caption{Parameter settings for various privacy budgets \( \epsilon \) on CIFAR-10\(_2\)}\vspace{-3mm}
  \label{tab:dp_params2}
  \begin{tabular}{c|c c c c c c c c}
    \toprule
    $\epsilon$ & 0.6 & 0.8 & 1 & 2 & 3 & 4 & 6 & 8 \\
    $\sigma$ & 1.20 & 1.20 & 1.20 & 0.80 & 0.80 & 0.50 & 0.20 & 0.20 \\
    $l$      & 2 & 6 & 5 & 6 & 0 & 0 & -- & -- \\
    \bottomrule
  \end{tabular}
\end{table}
It is worth noting that when $\epsilon \in \{6.0, 8.0\}$, we have $\mathcal{S}_2 = \emptyset$, making the choice of $l$ irrelevant. We therefore omit these results.

For evaluation, each configuration is run 10 times with different random seeds. Performance is averaged over the final 10 epochs, and we report the mean accuracy, standard deviation, and the average of per‑class accuracy, respectively. 
\subsection{Per-Class Test Accuracy}\label{appendix:per-class-acc}
Let $D_y$ denote a set of test samples belonging to class $y$, and $f$ be a classifier. Then the per-class accuracy is defined as
\[
ACC_y = \sum_{x \in D_y}\frac{\mathbb{I}(f(x)=y)}{|D_y|}.
\]
This subsection details the per-class accuracy results. \Cref{tab:cifar10_1(cls)} corresponds to the CIFAR-10$_1$ dataset, while \cref{tab:cifar10_2(cls)} corresponds to CIFAR-10$_2$.
\begin{table}[!htb]
\centering
\small
\setlength{\tabcolsep}{3pt}
\caption{Per-Class Test Accuracy (\%) on $\text{CIFAR-10}_1$ with Different Privacy Budgets.}\vspace{-3mm}
\label{tab:cifar10_1(cls)}
\begin{tabular}{|c|ccccc|ccccc|ccccc|}
\hline
\textbf{Budget} & \multicolumn{5}{c|}{\textbf{RR}} & \multicolumn{5}{c|}{\textbf{RRWithPrior}} & \multicolumn{5}{c|}{\textbf{Ours}} \\
\hline
\multirow{2}{*}{$\epsilon=0.6$} & 68.26,& 78.61,& 21.54,& 29.36,& 29.94,  & 73.33,& 70.91,& 29.90,& 11.39,& 28.00,  & 66.59,& 84.94,& 24.96,& 23.97,& 25.07,  \\
&42.31,& 1.56,& 13.05,& 0.85,& 5.95&48.14,& 0.00,& 0.00,& 0.00,& 0.00&51.92,& 0.00,& 7.29,& 0.00,& 0.42\\\hline
\multirow{2}{*}{$\epsilon=0.8$} & 69.65,& 79.94,& 38.91,& 29.20,& 30.19,  & 82.38,& 80.22,& 34.00,& 37.13,& 28.63,  & 71.65,& 80.46,& 37.95,& 33.57,& 36.99,  \\
& 52.94,& 8.77,& 32.66,& 6.05,& 24.19&68.35,& 0.00,& 0.00,& 0.00, &0.00&57.21,& 2.56,& 20.83,& 2.55,& 10.25\\\hline
\multirow{2}{*}{$\epsilon=1.0$} & 76.41,& 79.54,& 48.46,& 37.97,& 39.69,  & 77.70,& 84.28,& 59.66,& 51.41,& 46.22,  & 78.11,& 86.38,& 51.77,& 41.28,& 42.84,  \\
&58.47,& 19.61,& 48.13,& 11.12,& 40.69&64.89,& 0.00,& 0.00,& 0.00,& 0.00&65.15,& 10.27,& 27.55,& 3.75,& 25.94\\\hline
\multirow{2}{*}{$\epsilon=2.0$} & 89.53,& 96.58,& 73.18,& 67.34,& 73.16,  & 91.37,& 98.31,& 80.55,& 74.42,& 81.33,  & 90.57,& 97.23,& 76.82,& 68.99,& 74.34,  \\
&79.40,& 67.62,& 71.27,& 68.10,& 73.40&81.72,& 0.00,& 14.95,& 0.00, &8.13&79.20,& 61.13,& 67.64,& 59.53,& 64.82\\\hline
\multirow{2}{*}{$\epsilon=3.0$} & 93.24,& 98.77,& 86.12,& 80.28,& 85.56,  & 94.00,& 98.74,& 87.19,& 81.69,& 87.54,  & 93.02,& 98.84,& 85.90,& 79.72,& 85.81,  \\
&85.30,& 82.19,& 95.09,& 80.21,& 85.34&86.90,& 8.56,& 77.58,& 25.28, &61.01&85.86,& 82.13,& 85.35,& 80.21,& 84.06\\\hline
\multirow{2}{*}{$\epsilon=4.0$} & 94.80,& 99.13,& 89.18,& 84.93,& 88.81,  & 95.70,& 99.09,& 90.14,& 85.64,& 89.62,  & 94.86,& 99.18,& 89.38, &85.58,& 88.98,  \\
&87.81,& 86.29,& 90.37,& 83.59,& 88.80&88.36,& 87.81,& 91.14,& 85.61, &88.77&88.39,& 86.52,& 90.58,& 84.31,& 90.22\\\hline
\multirow{2}{*}{$\epsilon=6.0$} & 96.32,& 99.24,& 91.22,& 87.13,& 90.31,  & 96.39,& 99.26,& 91.98,& 88.43,& 91.46,  & 96.24,& 99.31,& 91.34, &87.43,& 90.26,  \\
&89.54,& 89.09,& 91.23,& 86.07,& 90.69&89.75,& 89.80,& 92.69,& 87.30, &91.03&89.65,& 89.25,& 91.86,& 85.77,& 91.08\\\hline
\multirow{2}{*}{$\epsilon=8.0$} & 96.48,& 99.26,& 91.39,& 86.87,& 90.77,  & 96.68,& 99.19,& 91.94,& 88.38,& 91.41,  & 96.25,& 99.33, &91.68, &88.15,& 91.17,  \\
&89.61,& 89.41,& 92.55,& 86.66,& 91.15&90.17,& 87.02,& 92.85,& 87.02,& 91.22&89.09,& 88.81,& 92.85,&86.64,& 91.32\\\hline
\multirow{2}{*}{$\epsilon=\infty$} & 96.28,& 99.29, &91.55,& 87.63,& 90.36,     & 96.81,& 99.34,& 91.87,& 88.77,& 91.32,     & 96.79,& 99.19,& 91.94,& 88.62,& 91.63, 
\\
&89.46,& 89.58,& 92.91,& 87.09,& 91.22&89.87,& 89.59,& 92.76,& 86.43, &90.90&89.84,& 89.84,& 92.76,& 87.03,& 91.14\\
\hline
\end{tabular}
\end{table}

\begin{table}[H]
\centering
\setlength{\tabcolsep}{3pt}
\caption{Per-Class Test Accuracy (\%) on $\text{CIFAR-10}_2$ with Different Privacy Budgets}\vspace{-3mm}
\label{tab:cifar10_2(cls)}
\begin{tabular}{|c|ccccc|ccccc|ccccc|}
\hline
\textbf{Budget} & \multicolumn{5}{c|}{\textbf{RR}} & \multicolumn{5}{c|}{\textbf{RRWithPrior}} & \multicolumn{5}{c|}{\textbf{Ours}} \\
\hline
\multirow{2}{*}{$\epsilon=0.6$} & 66.75,& 83.70,& 14.82,& 24.02,& 28.66,& 64.22,& 80.82,& 27.67, &27.08,& 31.60,  & 69.38,& 75.47,& 22.73, &28.48,& 33.65,  \\
&  52.25,& 1.52,& 0.86,& 1.24,& 0.50& 41.02,& 0.00,& 0.00,& 0.00,& 0.00 &47.64,& 0.02,& 0.20,& 0.33,& 0.32\\\hline
\multirow{2}{*}{$\epsilon=0.8$} & 70.24,& 74.33,& 39.97,& 32.83,& 34.34,  & 63.04,& 81.24, &46.45,& 45.17,& 46.77,  & 74.63,& 82.89,& 38.50,& 32.78,& 38.50, \\
& 57.08,& 5.62,& 7.90,& 4.97,& 3.55 & 52.14,& 0.00,& 0.00,& 0.00,& 0.00 &57.90,& 2.00, &4.56,& 0.14,& 0.50\\\hline
\multirow{2}{*}{$\epsilon=1.0$} & 76.41,& 79.54,& 48.46,& 37.39,& 39.69,  & 77.70,& 84.28, &59.66,& 51.41,& 46.22,  & 78.11,& 86.38,& 51.77,& 41.28,& 42.84, \\
& 58.47,& 19.61,& 48.13,& 11.12,& 40.69 & 64.89,& 0.00,& 0.00,& 0.00,& 0.00 &65.15,& 10.27, &27.55,& 3.75,& 25.94\\\hline
\multirow{2}{*}{$\epsilon=2.0$} & 91.09, &97.15,& 75.01,& 68.89,& 73.93,  & 92.65,& 98.36,& 80.58,& 75.03,& 80.78,  & 91.84,& 97.84,& 77.85,& 71.64,& 78.30,  \\
&79.86,& 58.75,& 39.26,& 60.04,& 35.60 &81.26,& 0.00,& 0.00,& 0.00,& 0.00 & 79.55,& 33.95,& 17.00,& 37.36,& 15.55\\\hline
\multirow{2}{*}{$\epsilon=3.0$} & 93.74,& 99.16,& 86.30,& 80.50, &86.09,  & 94.68,& 99.49,& 88.59,& 83.60,& 88.21,  & 93.59, &9892,& 85.86,& 80.86,& 85.40,  \\
& 85.73,& 75.28,& 59.98,& 76.20,& 73.70& 86.82,& 8.03,& 0.00,& 0.00,& 0.00& 85.41,& 75.03,& 58.06,& 76.64,& 69.47\\\hline
\multirow{2}{*}{$\epsilon=4.0$} & 95.01,& 99.65,& 91.42,& 87.40,& 91.65,  & 95.35,& 99.47,& 90.50, &85.69,& 89.64,  & 95.21,& 99.33,& 89.90, &85.34,& 88.87,  \\
&89.73, &85.27,& 83.24,& 84.49,& 84.50&88.77,& 58.22,& 28.48,& 58.63,& 45.37&88.03,& 82.18,& 69.46,& 83.20,& 78.55\\\hline
\multirow{2}{*}{$\epsilon=6.0$} & 96.41,& 99.62,& 91.08,& 87.40,& 90.65,  & 96.39,& 99.26,& 91.98,& 88.43,& 91.46,  & 97.29,& 99.52,& 91.85,& 87.32,& 91.42,  \\
&89.73,& 85.05,& 81.32,& 84.53,& 83.03&89.75,& 89.80,& 92.69,& 87.30, &91.03&89.88,& 86.57, &81.80,& 83.67,& 83.65\\\hline
\multirow{2}{*}{$\epsilon=8.0$} & 96.51,& 99.60,& 91.42,& 87.72,& 91.60,  & 96.68,& 99.19,& 91.94,& 88.38,& 91.41,  & 96.95,& 99.74,& 92.27,& 88.33,& 91.50,  \\
& 89.37,& 85.27,& 83.22,& 84.49,& 84.55&90.17,& 87.02,& 92.85,& 87.02, &91.22&89.70,& 86.20,& 83.56,& 84.89,& 83.80\\\hline
\multirow{2}{*}{$\epsilon=\infty$} & 96.81,& 99.75,& 91.43,& 88.00,& 91.43,  & 97.23,& 99.72,& 92.17,& 88.56,& 92.46,  & 97.07,& 99.63,& 91.98,& 88.28,& 91.64,  \\
& 89.74,& 85.53,& 81.58,& 84.93,& 81.60&90.24,& 87.42,& 85.02,& 85.56,& 83.22&89.90,& 86.78,& 84.20,& 84.44,& 83.02\\
\hline
\end{tabular}
\end{table}

\end{document}